\documentclass{article}

\usepackage[preprint]{neurips_2025}

\usepackage[utf8]{inputenc} 
\usepackage[T1]{fontenc}    
\usepackage{hyperref}       
\usepackage{url}            
\usepackage{booktabs}       
\usepackage{amsfonts}       
\usepackage{nicefrac}       
\usepackage{microtype}      
\usepackage{xcolor}         
\usepackage{amsthm}
\usepackage{amsmath, amssymb}

\newtheorem{theorem}{Theorem}
\newtheorem{corollary}{Corollary}
\newtheorem{proposition}{Proposition}

\newtheorem{remark}{Remark}
\newtheorem{definition}{Definition}

\title{Closed-Form Robustness Bounds for Second-Order Pruning of Neural Controller Policies}

\author{%
  Maksym Shamrai\\
  Institute of Mathematics\\
  National Academy of Sciences of Ukraine\\
  Kyiv, Ukraine\\
  \texttt{m.shamrai@imath.kiev.ua}
}

\begin{document}

\maketitle

\begin{abstract}
Deep neural policies have unlocked agile flight for quadcopters, adaptive grasping for manipulators, and reliable navigation for ground robots, yet their millions of weights conflict with the tight memory and real-time constraints of embedded microcontrollers.
Second-order pruning methods, such as \emph{Optimal Brain Damage} (OBD) and its
variants, including \emph{Optimal Brain Surgeon} (OBS) and the recent
\textsc{SparseGPT}, compress networks in a single pass by leveraging the local Hessian, achieving far higher sparsity than magnitude thresholding.
Despite their success in vision and language, the consequences of such weight removal on \emph{closed-loop} stability, tracking accuracy, and safety have remained unclear.

We present the first mathematically rigorous robustness analysis of
second-order pruning in nonlinear discrete-time control.
The system evolves under a continuous transition map, while the controller is an \(L\)-layer multilayer perceptron with ReLU-type activations that are \emph{globally} \(1\)-Lipschitz. 
Pruning the weight matrix of layer \(k\) replaces \(W_k\) with \(W_k+\delta W_k\), producing the perturbed parameter vector \(\widehat{\Theta}=\Theta+\delta\Theta\) and the pruned policy \(\pi(\cdot;\widehat{\Theta})\).

For every input state \(s\in X\) we derive the closed-form inequality
\[
\|\pi(s;\Theta)-\pi(s;\widehat{\Theta})\|_2
\le
C_k(s)\,\|\delta W_k\|_2,
\]
where the constant \(C_k(s)\) depends \emph{only} on unpruned spectral norms and biases, and can be
evaluated in closed form from a single forward pass.

The derived bounds specify, prior to field deployment, the maximal admissible pruning magnitude compatible with a prescribed control-error threshold.
By linking second-order network compression with closed-loop performance guarantees, our work narrows a crucial gap between modern deep-learning tooling and the robustness demands of safety-critical autonomous systems.
\end{abstract}

\section{Introduction}\label{sec:intro}

Modern autonomous systems, from agile quadcopters to embodied
household robots, are increasingly controlled by
\emph{neural policies} that map high-dimensional sensor data directly to
control actions.
Rich function classes such as multilayer perceptrons (MLPs),
convolutional networks, or
\emph{Vision–Language–Action} (VLA) models
achieve impressive closed-loop performance when trained
by reinforcement learning (RL) or behaviour cloning
\cite{mnih2015human, black2024pi_0}.
Yet deploying those policies on size, weight, and power-constrained
hardware demands aggressive \emph{model compression}.
Among the most practical compression techniques are
\emph{second-order pruning} algorithms: Optimal Brain Damage (OBD)
\cite{lecun1989optimal}, Optimal Brain Surgeon (OBS)
\cite{hassibi1992second}, and their recent large-model successor
\emph{SparseGPT}~\cite{frantar2023sparsegpt}, which remove weights that
minimise an analytically estimated activation loss.

While second-order pruning is widely used for
large language models, its impact on the \emph{closed-loop behaviour} of
a control system is poorly investigated.
A small perturbation of the weights can propagate through the policy
network, alter the control signal, and ultimately degrade safety or
task performance.
Precise guarantees on how much pruning a controller can tolerate are
therefore crucial for safety-critical applications
\cite{gu2022review, virmauxLipschitzRegularityDeep2018, zhangRethinkingLipschitzNeural2022, nadizar2021effects}. To the best of our knowledge, \emph{no prior work provides a rigorous,
closed–form upper bound on the control error induced by
second–order pruning} in nonlinear systems.

We provide the first \emph{robustness analysis} of OBD/OBS-style weight
pruning for deterministic nonlinear control systems.
Our goal is to bound, in closed form, the deviation of the pruned control signal.

\paragraph{Contributions.}
\begin{enumerate}
  \item \textbf{Formal robustness framework.}  
        Section~\ref{sec:setup} casts pruning as a perturbation of the
        parameter vector~$\Theta$ and formulates the \emph{pruning
        robustness problem} via the control–error bound
        \eqref{eq:control-bound}.
  \item \textbf{Tight single–layer bound.}
        Theorem~\ref{thm:robust} proves that for any 1–Lipschitz
        (ReLU–type) activation the deviation of a pruned layer~$k$ is
        \[
            \bigl\|\pi(s;\Theta)-\pi(s;\widehat{\Theta})\bigr\|_{2}
            \;\le\;
            C_{k}(s)\,\|\delta W_{k}\|_{2},
        \]
        where the constant $C_{k}(s)$ depends only on \emph{unpruned}
        weights, biases and the input norm.
  \item \textbf{Extension to multiple layers.}
        Corollary~\ref{cor:multi} extends the result to an arbitrary set
        of pruned layers in an \emph{additive} fashion, yielding the
        computable bound
        $B_\pi(\delta\Theta)=\sum_{k\in S}C_{k,\max}\|\delta W_{k}\|_{2}$.
  \item \textbf{Practical implications.}
        The bounds can be evaluated \emph{offline} from a forward pass
        and spectral norms alone, enabling principled trade–offs
        between compression ratio and control robustness without
        repeated interaction with the physical system.
\end{enumerate}

\section{Formal Problem Setup}\label{sec:setup}

\subsection{Deterministic nonlinear control problem}

\begin{definition}[Controlled dynamics]\label{def:dynamics}
Let the \textbf{state space} be $X \subset \mathbb{R}^n$ and the
\textbf{action space} be $U \subset \mathbb{R}^m$.
A trajectory $\{x_t\}_{t\ge 0}$ evolves according to the discrete--time
difference equation
\[
        x_{t+1}\;=\;f\!\bigl(x_t,u_t\bigr),\qquad
        t=0,1,2,\dots ,
\]
where the transition map
$f : X\times U \rightarrow X$ is continuous, ensuring that a unique
trajectory exists for every admissible control sequence
$\{u_t\}_{t\ge 0}$ and initial state $x_0\in X$.
\end{definition}

\begin{definition}[Parametric neural policy]\label{def:policy}
Let $L\in\mathbb{N}$ be the number of affine layers.
Collect all weights and biases in the parameter vector
\[
        \Theta
        \;=\;\bigl\{W_\ell,b_\ell\bigr\}_{\ell=1}^{L}\in\mathbb{R}^{q}.
\]
The \textbf{policy}
$\pi(\,\cdot\,;\Theta):X\rightarrow U$ is the neural network obtained by
composing these affine layers with point\-wise, $1$--Lipschitz
activations~$\sigma$:
\[
\pi(x;\Theta)
\;:=\;
(\sigma\!\circ\! A_L)\circ\cdots\circ(\sigma\!\circ\! A_1)(x),\qquad
A_\ell(z)\;=\;W_\ell z + b_\ell .
\]
The control applied at time~$t$ is $u_t=\pi(x_t;\Theta)$.
\end{definition}

\begin{definition}[Discounted return]\label{def:return}
Fix an initial distribution $x_0\sim\mu$, a bounded reward
$r:X\times U\rightarrow\mathbb{R}$, and a discount factor
$\gamma\in(0,1)$.  The expected return of policy $\pi(\cdot;\Theta)$ is
\[
        J(\Theta)
        \;:=\;
        \mathbb{E}_{x_0\sim\mu}
        \bigl[
            \textstyle\sum_{t=0}^{\infty}
            \gamma^{t}\,r\bigl(x_t,\pi(x_t;\Theta)\bigr)
        \bigr].
\]
An \textbf{optimal policy} is any maximiser
$\Theta^\star\in\arg\max_{\Theta\in\mathbb{R}^{q}}J(\Theta)$.
\end{definition}

\subsection{Second--order (OBD) pruning criterion}

Throughout this paper we study the effect of pruning individual weights
in~$\Theta$ by the \emph{Optimal Brain Damage (OBD)} saliency
\cite{lecun1989optimal}.  Consider a single affine layer
$A_\ell(z)=W_\ell z+b_\ell$ and an input mini--batch
$X_\ell\in\mathbb{R}^{d\times n_\ell}$.
The change in pre--activation outputs caused by replacing
$W_\ell$ with $\widehat{W}_\ell=W_\ell+\delta W_\ell$ is measured by
\begin{equation}\label{eq:activ-loss}
        E_\ell(W_\ell,\widehat{W}_\ell)
        \;=\;
        \bigl\|W_\ell X_\ell-\widehat{W}_\ell X_\ell\bigr\|_F^{2}.
\end{equation}
A second--order Taylor expansion of $E_\ell$ around $W_\ell$ yields
\[
        E_\ell(W_\ell+\delta W_\ell)
        \;\approx\;
        \tfrac12
        \operatorname{vec}(\delta W_\ell)^{\!\top}
        H_\ell\,
        \operatorname{vec}(\delta W_\ell),
        \qquad
        H_\ell:=\nabla^{2}_{W_\ell}E_\ell \;.
\]
Pruning a \emph{single} weight $w_q\in\Theta$ (that is, setting
$\widehat{w}_q=0$)
gives the OBD saliency
\begin{equation}\label{eq:obd-saliency}
        \Delta E_q
        \;=\;
        \frac12\,\frac{w_q^{2}}{(H_\ell^{-1})_{qq}},
\end{equation}
which serves as an importance score for that weight.
Aggregating \eqref{eq:obd-saliency} over the layers to which OBD is
applied produces a pruned parameter vector
$\widehat{\Theta}\in\mathbb{R}^{q}$ and the perturbation
$\delta\Theta:=\widehat{\Theta}-\Theta$.

\subsection{Pruning robustness problem}

Our objective is to quantify \textbf{how sensitive the closed--loop
behaviour is to the OBD perturbation}~$\delta\Theta$.
Specifically, we seek a computable bound
$B_{\pi}:\mathbb{R}^{q}\rightarrow\mathbb{R}_{+}$ such that
\begin{equation}\label{eq:control-bound}
        \bigl\|\pi(x;\Theta)-\pi(x;\widehat{\Theta})\bigr\|_{2}
        \;\le\;
        B_{\pi}(\delta\Theta),
        \qquad
        \forall\,x\in X .
\end{equation}
Although inequality \eqref{eq:control-bound} applies to \emph{any}
feed--forward network, we focus on the control setting because the
resulting guarantees translate directly into performance and safety
margins for autonomous systems.

\begin{remark}
For clarity, we treat each weight matrix $W_\ell$ and bias $b_\ell$ as a
distinct block within~$\Theta$.  Hence any scalar weight~$w_q$ appearing
in \eqref{eq:obd-saliency} is an element of~$\Theta$, and replacing
$w_q$ by $0$ modifies $\Theta$ exactly as required in the control
bound~\eqref{eq:control-bound}.
\end{remark}

\section{Background}\label{sec:background}

All robustness estimates in this paper rest on \emph{Lipschitz control}  
of the policy network.  We collect the required analytic facts in this
section.

\subsection{Lipschitz mappings}

\begin{definition}[Global Lipschitz continuity]\label{def:lipschitz}
A function $f:\mathbb{R}^{n}\!\to\!\mathbb{R}^{m}$ is called
\textbf{$L$–Lipschitz} with respect to the Euclidean norm if
\begin{equation}\label{eq:lipschitz}
        \|f(x)-f(y)\|_{2}
        \;\le\;
        L\,\|x-y\|_{2},
        \qquad
        \forall\,x,y\in\mathbb{R}^{n}.
\end{equation}
The smallest such constant is denoted $L(f)$.
\end{definition}

\begin{theorem}[Rademacher \cite{evans2018measure}]\label{thm:rademacher}
Every locally Lipschitz map
$f:\mathbb{R}^{n}\!\to\!\mathbb{R}^{m}$ is differentiable almost
everywhere and
\(
        L(f)
        \;=\;
        \operatorname*{ess\,sup}_{x}\,\|Df(x)\|_{2},
\)
where $\|\cdot\|_{2}$ is the operator norm induced by $\ell_{2}$.
\end{theorem}


\subsection{Lipschitz multilayer perceptrons}

\begin{definition}[MLP with $1$–Lipschitz activations]\label{def:mlp}
Fix $L\in\mathbb{N}$.
Let $\sigma:\mathbb{R}\!\to\!\mathbb{R}$ satisfy
$|\sigma(a)-\sigma(b)|\le|a-b|$.  
An \textbf{$L$-layer MLP} is the map
\[
        f(x;\Theta)
        := (\sigma\!\circ\!A_{L})\circ\cdots\circ(\sigma\!\circ\!A_{1})(x),
        \qquad
        A_{\ell}(z)=W_{\ell}z+b_{\ell},
\]
with parameters
$\Theta=\{W_{\ell},b_{\ell}\}_{\ell=1}^{L}$ and
$W_{\ell}\in\mathbb{R}^{d_{\ell}\times d_{\ell-1}}$.
\end{definition}

\begin{proposition}[Spectral-norm Lipschitz bound~\cite{barbaraRobustReinforcementLearning2025}]
\label{prop:mlp-lip}
For the MLP of Definition~\ref{def:mlp},
\begin{equation}\label{eq:mlp-lip}
        L\!\bigl(f(\,\cdot\,;\Theta)\bigr)
        \;\le\;
        \prod_{\ell=1}^{L}\|W_{\ell}\|_{2}.
\end{equation}
\end{proposition}

\paragraph{Interpretation for pruning analysis.}
Equation~\eqref{eq:mlp-lip} expresses the \emph{global} Lipschitz
constant of a neural policy directly in terms of the spectral
(operator-norm) factors that appear in the OBD parameter vector
$\Theta$.  Hence perturbing any weight matrix
$W_{k}\mapsto W_{k}+\delta W_{k}$ alters both
\(
        L\!\bigl(f(\cdot;\Theta)\bigr)
\)
and the saliency \eqref{eq:obd-saliency}, allowing us to translate the
activation-level OBD error into a control-signal bound
\eqref{eq:control-bound} in later sections.







\section{Results}\label{sec:results}

\begin{proposition}[Non-expansiveness of ReLU-type activations]
\label{prop:nonexpansive-activ}
Let $\varphi:\mathbb{R}\!\to\!\mathbb{R}$ be an activation satisfying  
\begin{enumerate}
    \item \emph{zero anchor:}\; $\varphi(0)=0$,
    \item \emph{unit Lipschitz:}\; $|\varphi(a)-\varphi(b)|\le |a-b|$ for all $a,b\in\mathbb{R}$.
\end{enumerate}
Define the component-wise map 
\[
    \sigma:\mathbb{R}^n\!\to\!\mathbb{R}^n,\qquad
    \sigma(x)_i \;=\; \varphi(x_i),\; i=1,\dots,n .
\]
Then $\sigma$ is \emph{$\ell_2$–non-expansive}:
\begin{equation}\label{eq:nonexp-2norm}
        \|\sigma(x)\|_2 \;\le\; \|x\|_2 , 
        \qquad \forall\,x\in\mathbb{R}^n .
\end{equation}
\end{proposition}

\begin{proof}
For any $x\in\mathbb{R}^n$ we have, by the scalar Lipschitz property with $b=0$,
$|\varphi(x_i)| \le |x_i|$.  Squaring, summing and taking the square
root yields
\[
        \|\sigma(x)\|_2^2 
        \;=\; \sum_{i=1}^n \varphi(x_i)^2
        \;\le\; \sum_{i=1}^n x_i^2
        \;=\; \|x\|_2^2,
\]
from which \eqref{eq:nonexp-2norm} follows.
\end{proof}

Condition~\eqref{eq:nonexp-2norm} is satisfied by many ReLU-type activations that are ubiquitous in deep learning:
\begin{itemize}
\item \textbf{ReLU} $\varphi(x)=\max\{0,x\}$
      \cite{nair2010rectified};
\item \textbf{Leaky-ReLU} $\varphi(x)=\max\{x,\alpha x\}$ with
      $0<\alpha\le 1$ \cite{maas2013rectifier};
\item \textbf{PReLU} (parametric ReLU) with learnable
      $\alpha\!\in\!(0,1]$ \cite{he2015delving};
\item \textbf{ELU} $\varphi(x)=\max\{x,\alpha(e^{x}-1)\}$ for
      $0<\alpha\le 1$ \cite{clevert2015fast};
\item \textbf{GELU} (Gaussian-error linear unit), a smooth approximation
      to ReLU whose derivative is bounded by~$1$
      \cite{hendrycks2016gaussian}.
\end{itemize}
All these functions obey $\varphi(0)=0$ and have slope bounded by~$1$,
hence are $1$-Lipschitz, therefore Proposition~\ref{prop:nonexpansive-activ} applies.

\begin{theorem}[Robustness of an OBD-pruned policy]\label{thm:robust}
Let $\pi(\,\cdot\,;\Theta)$ be the $L$-layer MLP controller
\[
  x_0=s,\qquad
  x_{\ell}= \sigma\bigl(W_{\ell}x_{\ell-1}+b_\ell\bigr),
  \; \ell=1,\dots,L,\quad
  \Pi(s;\Theta)=x_L,
\]
with ReLU-type activations~$\sigma_l$ and
weights $\Theta=\{W_\ell,b_\ell\}_{\ell=1}^{L}$.
Suppose layer $k$ is pruned, giving
$\widehat{W}_k=W_k+\delta W_k$ and
$\widehat{\Theta}=\Theta+\delta\Theta$.
Then, for every input state $s \in X$,
\begin{align}\label{eq:robust-bound}
  \bigl\|
    \pi(s;\Theta)-\pi(s;\widehat{\Theta})
  \bigr\|_2
   & \;\le\;
    \|\delta W_k \|_2
  \Big(
    \lVert s\rVert_2
    \prod_{\substack{\ell=1\\\ell\neq k}}^{L}\lVert W_\ell\rVert_2
    +
    \sum_{i=1}^{k-1} \prod_{\substack{\ell=i+1\\\ell\neq k}}^{L}\lVert W_\ell\rVert_2 \| b_i \|_2
  \Big)
  \\ & \;\le\;
  \|\delta W_k \|_2
  \underbrace{
  \Big(
    \sup_{s\in X}\lVert s\rVert_2
    \prod_{\substack{\ell=1\\\ell\neq k}}^{L}\lVert W_\ell\rVert_2
    +
    \sum_{i=1}^{k-1} \prod_{\substack{\ell=i+1\\\ell\neq k}}^{L}\lVert W_\ell\rVert_2 \| b_i \|_2
  \Big)
  }_{\displaystyle =:\,C_{\max}}.
\end{align}
Consequently, the control-error bound in~\eqref{eq:control-bound} can be
chosen as $B_\pi(\delta\Theta)=C_{\max}\lVert\delta W_k\rVert_2$.
\end{theorem}

\begin{proof}
Because $\sigma_l$ are 1-Lipschitz,
\(
  \lVert \sigma_l(u)-\sigma_l(v)\rVert_2\le\lVert u-v\rVert_2.
\)
Write $\delta x_\ell:=x_\ell-\widehat{x}_\ell$, then for the last layer
\[
\| \pi(s;\Theta)-\pi(s;\widehat{\Theta}) \|_2
  =
  \| x_L-\widehat{x}_L \|_2
  =
  \lVert\delta x_L\rVert_2,
\]
and 
\begin{align*}
  \lVert\delta x_L\rVert_2
  & =
  \lVert
    \sigma_L(W_Lx_{L-1}+b_L)-
    \sigma_L(W_L\widehat{x}_{L-1}+b_L)
  \rVert_2 \\
  & \le
  \lVert
    W_Lx_{L-1} - W_L\widehat{x}_{L-1}
  \rVert_2
  \le
  \lVert W_L\rVert_2\lVert\delta x_{L-1}\rVert_2 \\
  & \Rightarrow \| \delta x_L \|_2 \le \lVert W_L\rVert_2\lVert\delta x_{L-1}\rVert_2.
\end{align*}
Iterating the argument down to layer $k$ yields
\begin{align*}
  \lVert\delta x_L\rVert_2
  & \le
  \Bigl(
    \prod_{\ell=k+1}^{L}\lVert W_\ell\rVert_2
  \Bigr)
  \lVert
    W_kx_{k-1}-\widehat{W}_kx_{k-1}
  \rVert_2 \\
  & =
  \Bigl(
    \prod_{\ell=k+1}^{L}\lVert W_\ell\rVert_2
  \Bigr)
  \|\delta W_k x_{k-1}\|_2 \\
  & \le
  \Bigl(
    \prod_{\ell=k+1}^{L}\lVert W_\ell\rVert_2
  \Bigr)
  \|\delta W_k \|_2 \| x_{k-1}\|_2.
\end{align*}

Because $\sigma_l$ are ReLU-type by Proposition~\ref{prop:nonexpansive-activ}
\begin{align*}
\lVert x_{k-1}\rVert_2
& = \| \sigma_{k-1}(W_{k-1}x_{k-2} + b_{k-1}) \|_2 \\
& \le \| W_{k-1}x_{k-2} + b_{k-1} \|_2 \\
& \le \| W_{k-1} \|_2 \| x_{k-2} \|_2 + \| b_{k-1} \|_2 \\ 
& \le \dots \le \\
& \le \lVert s\rVert_2 \prod_{\ell=1}^{k-1}\lVert W_\ell\rVert_2 
+ 
\sum_{i=1}^{k-2} \prod_{\ell=i + 1}^{k - 1}\lVert W_\ell\rVert_2 \| b_{i} \|_2 + \| b_{k-1} \|_2.
\end{align*}

Substituting gives

\begin{align*}
    \| \pi(s;\Theta)-\pi(s;\widehat{\Theta}) \|_2 
    & \le 
    \Bigl(
    \prod_{\ell=k+1}^{L}\lVert W_\ell\rVert_2
  \Bigr)
  \|\delta W_k \|_2
  \Big(
  \lVert s\rVert_2 \prod_{\ell=1}^{k-1}\lVert W_\ell\rVert_2 
+ 
\sum_{i=1}^{k-2} \prod_{\ell=i + 1}^{k - 1}\lVert W_\ell\rVert_2 \| b_{i} \|_2 + \| b_{k-1} \|_2
  \Big) \\
  & \le
\|\delta W_k \|_2
  \Big(
    \lVert s\rVert_2
    \prod_{\substack{\ell=1\\\ell\neq k}}^{L}\lVert W_\ell\rVert_2
    +
    \sum_{i=1}^{k-1} \prod_{\substack{\ell=i+1\\\ell\neq k}}^{L}\lVert W_\ell\rVert_2 \| b_i \|_2
  \Big),
\end{align*}
which completes the bound \ref{eq:robust-bound}.
\end{proof}

\begin{corollary}[Robustness under pruning \emph{multiple} layers]
\label{cor:multi}
Let the assumptions of Theorem~\ref{thm:robust} hold and let
\(
        S=\{k_{1},\dots,k_{m}\}\subseteq\{1,\dots,L\}
\)
be an index set of layers pruned by OBD.
For every \(k\in S\) write
\(
        \widehat{W}_{k}=W_{k}+\delta W_{k}
\)
and
\(
        \widehat{\Theta}=\Theta+\delta\Theta
\)
with
\(
        \delta\Theta=\{\delta W_{k}\}_{k\in S}.
\)
Define, for each \(k\in S\) and input state \(s\in X\),
\begin{equation}\label{eq:Ck}
    C_{k}(s)
    \;:=\;
    \lVert s\rVert_{2}
    \!\!\!\prod_{\ell\,\neq\, k}\!\!\lVert W_{\ell}\rVert_{2}
    +
    \sum_{i=1}^{k-1}
        \Bigl(
            \!\!\prod_{\substack{\ell=i+1 \\ \ell\neq k}}^{L}
            \lVert W_{\ell}\rVert_{2}
        \Bigr)
        \lVert b_{i}\rVert_{2},
    \qquad
    C_{k,\max}:=\sup_{s\in X}C_{k}(s).
\end{equation}
Then, for every \(s\in X\),
\begin{align}\label{eq:multi-bound}
    \bigl\|\pi(s;\Theta)-\pi(s;\widehat{\Theta})\bigr\|_{2}
    &\;\le\;
      \sum_{k\in S}
      \lVert\delta W_{k}\rVert_{2}\;C_{k}(s)
      \\[2mm]
    &\;\le\;
      \sum_{k\in S}
      \lVert\delta W_{k}\rVert_{2}\;C_{k,\max}.
\end{align}
Consequently, a valid control-error budget in
\eqref{eq:control-bound} is
\(
        B_{\pi}(\delta\Theta)
        =\sum_{k\in S}
          C_{k,\max}\,\lVert\delta W_{k}\rVert_{2}.
\)
\end{corollary}

\begin{proof}
Order the pruned layers as \(k_{1}<\cdots<k_{m}\) and construct an
intermediate parameter sequence
\(
        \Theta^{0}:=\Theta,\;
        \Theta^{j}:=\Theta^{j-1}
        +\delta W_{k_{j}},\;j=1,\dots,m,
\)
so that \(\Theta^{m}=\widehat{\Theta}\).

Applying Theorem~\ref{thm:robust} to the pair
\((\Theta^{j-1},\Theta^{j})\) and layer \(k_{j}\) gives
\[
    \bigl\|
        \pi(s;\Theta^{j-1})-\pi(s;\Theta^{j})
    \bigr\|_{2}
    \;\le\;
    \lVert\delta W_{k_{j}}\rVert_{2}\,C_{k_{j}}(s).
\]
Summing these $m$ inequalities and using the triangle inequality yields
the first bound in \eqref{eq:multi-bound}. Taking the supremum over
\(s\) establishes the second bound.
\end{proof}

Corollary~\ref{cor:multi} shows that control robustness degrades
\emph{additively} with respect to the spectral-norm perturbations
\(\{\delta W_{k}\}_{k\in S}\).  The constants \(C_{k}(s)\) depend only
on the unpruned weights, biases, and the input magnitude, making them
computable \emph{before} pruning takes place.  In practice one may
evaluate the tighter, state-dependent bound \(C_{k}(s)\) on a validation
set or deploy the uniform constant \(C_{k,\max}\) for worst-case
guarantees.

\section{Conclusion}\label{sec:conclusion}

We have provided the first closed-form robustness guarantees for
second-order network pruning in deterministic, nonlinear
discrete-time control.
By combining classical Lipschitz bounds with a layer-local analysis of
the OBD saliency, we showed that the deviation of the pruned control
signal is proportional to the spectral-norm perturbation introduced in
each affected layer, and that these deviations accumulate additively
across layers.
All constants appearing in the bounds depend only on unpruned spectral
norms, biases and input magnitude, so they can be evaluated offline
from a single forward pass and require no roll-outs, no retraining, no additional
hyperparameters.

\paragraph{Limitations.}
Our analysis is restricted to deterministic dynamics, ReLU-type
1-Lipschitz activations and spectral-norm estimates. It therefore
ignores stochastic disturbances, data-dependent curvature.
The bounds apply to the instantaneous control signal, not yet to the long-horizon value function or trajectory-level safety constraints.

\paragraph{Future directions.}
Promising extensions include

\begin{itemize}
    \item lifting the theory to stochastic or adversarial disturbances;
    \item deriving return-level or constraint-violation bounds via contraction arguments;
    \item using the bounds to devise \emph{robustness-aware} pruning schedules that maximize compression under a certified control-error budget;
\end{itemize}

Bridging these directions will further tighten the link between modern
network compression and the rigorous assurance required for autonomous
systems.

\section*{Acknowledgments}
The authors confirm that there is no conflict of interest and acknowledges financial support by the Simons Foundation grant (SFI-PD-Ukraine-00014586, M.S.) and the
project 0125U000299 of the National Academy of Sciences of Ukraine. We also express our gratitude to the Armed Forces of Ukraine for their protection, which has made this research possible.

\bibliographystyle{plain}
\bibliography{references}

\end{document}